\documentclass{article}
\usepackage{iclr2027_conference,times}

\usepackage{amsmath,amsfonts,bm}

\def\eqref#1{equation~\ref{#1}}

\def\1{\bm{1}}

\DeclareMathAlphabet{\mathsfit}{\encodingdefault}{\sfdefault}{m}{sl}
\SetMathAlphabet{\mathsfit}{bold}{\encodingdefault}{\sfdefault}{bx}{n}

\definecolor{cvprblue}{rgb}{0.21,0.49,0.74}
\definecolor{highlightred}{gray}{0.9}
\usepackage{pifont}
\usepackage[utf8]{inputenc} 
\usepackage[T1]{fontenc}    
\usepackage[pagebackref,breaklinks,colorlinks,citecolor=cvprblue]{hyperref}
\usepackage{amssymb}
\usepackage{amsthm}
\usepackage{mathtools}
\usepackage{booktabs}
\usepackage{multirow}
\usepackage{graphicx}
\usepackage{wrapfig}
\usepackage{cleveref}
\crefname{equation}{Eq.}{Eqs.}
\usepackage{algorithm}
\usepackage{algpseudocode}

\makeatletter
\renewcommand{\eqref}[1]{\textup{(\ref{#1})}}

\newtheorem{lemma}{Lemma}
\newtheorem{theorem}{Theorem}
\newtheorem{corollary}{Corollary}

\newcommand{\TEV}{\textsc{TEV}}

\newcommand{\pa}{\mathrm{pa}}
\newcommand{\Ch}{\mathrm{Ch}}
\newcommand{\depth}{\mathrm{d}}

\title{Where Draft Trees Lose Target Mass: Exit-Guided Speculative Decoding}

\iclrfinalcopy

\author{
Shijing Hu$^1$ \hspace{4em} Xuancheng Ren$^1$ \hspace{4em} Zhihui Lu$^1$\thanks{Corresponding author.} \hspace{5em} Pan Zhou$^{2}$ \\
\hspace{5em} $^1$ Fudan University \hspace{5em} $^2$ Singapore Management University \\
\texttt{sjhu24@m.fudan.edu.cn} \hspace{1em} \texttt{xcren25@m.fudan.edu.cn} \hspace{1em} 
\texttt{lzh@fudan.edu.cn} \\ 
\hspace{14em} \texttt{panzhou@smu.edu.sg}
}

\begin{document}

\maketitle

\begin{abstract} 
	Tree-based speculative decoding organizes draft tokens into a bounded tree so that multiple continuations can be verified in one target-model pass. However, the tree is constructed from draft-side scores while its usefulness is ultimately determined by the target model, creating a fundamental draft--target mismatch under a finite tree budget. We study two questions: once a draft tree is fixed, can a better exact verifier accept more draft tokens, and if not, how can target feedback improve the tree itself? We answer both through a target-flow view of the fixed draft tree. We identify a canonical \emph{exit law} that describes where target continuations leave the tree, and prove that one plus the resulting target coverage is a sharp upper bound on the expected output-block length, including the bonus token, of any exact path verifier. Moreover, every verifier attaining this ceiling must realize the same exit and bonus-token law; representative predraw-and-follow and sequential residual verifiers already attain it. This theory directly yields \emph{Tree Exit Verification} (TEV), an exact verifier that realizes the canonical law through one exit-node decision and one bonus-token decision, exposing a regular, level-parallel verification procedure. The same exit law also localizes where target probability is missing from the current tree. We use it as node-level feedback to train the drafter on its inference-time draft trees, reallocating finite tree budget toward target-relevant regions. Experiments across dialogue, code, and mathematical reasoning validate the fixed-tree equivalence, show that the \emph{Exit-Guided Draft-Tree Training} (ExitTrain) improves average output-block length by 13\%, and that TEV reduces verifier-stage latency by 15\%, yielding 14\% end-to-end speedup over DDTree. Together, our results separate the two remaining opportunities in bounded tree speculative decoding: \emph{better draft trees for higher acceptance, and more direct verification for lower latency}. Code is available at \url{https://github.com/hsj576/TEV}. 
\end{abstract}

\vspace{-4mm}
\section{Introduction}
\label{sec:introduction}
\vspace{-2mm}
 Large language models (LLMs) decode autoregressively, requiring an expensive target-model forward pass per token. Speculative decoding reduces this cost by using a lightweight draft model to propose future tokens that the target verifies in parallel~\citep{leviathan2023fast,chen2023accelerating}. 
While early methods draft a single sequence, recent systems organize many draft tokens into a  tree, allowing multiple possible continuations to share one target-model verification pass~\citep{miao2024specinfer,chen2024sequoia,li2024eagle2,li2025eagle}.
Tree speculative decoding can therefore be viewed as three stages:
the \emph{drafter} proposes draft tokens and their scores;
the \emph{tree builder} selects and organizes a limited number of them into a draft tree;
and the \emph{verifier} uses the target model to determine how far generation can safely follow a path in that tree while preserving the target distribution.
Existing systems implement verifier differently.
DDTree's predraw-and-follow verifier samples a target token at each tree node and follows a matching drafted child~\citep{ringel2026ddtree}, whereas EAGLE-style sequential residual verifier tests drafted children sequentially while updating the remaining target probability~\citep{li2025eagle}.
DFlash with DDTree provides one concrete realization of this drafter--builder--verifier pipeline~\citep{chen2026dflash,ringel2026ddtree}.

Under a fixed tree budget, the draft tree determines which future prefixes can be verified in a single target forward pass.
Among exponentially many possible prefixes, only a small number can appear in the tree. 
Since the tree must be built \emph{before} the target model evaluates it, practical builders rely on draft-model scores (e.g., confidence) to decide which branches to keep~\citep{li2024eagle2,li2025eagle,ringel2026ddtree}.
Yet the usefulness of those branches is ultimately determined by the target model.
When the draft and target prefer different continuations, limited tree budget may be spent on branches that look promising to the drafter but are rarely followed by the target, while target-likely branches remain insufficiently expanded in the tree.
Fig.~\ref{fig:teaser}(a) illustrates this \emph{draft--target tree-allocation gap}.
Although we use DDTree as our main instantiation, the same mismatch arises whenever draft-side information is used to construct a bounded draft tree before target verification.

This mismatch raises two fundamental questions. \emph{First}, once a draft tree is fixed, how much freedom remains in exact verification? Recent work like HSD~\citep{zhou2026hsd} shows that improved verification can increase acceptance for a sampled draft chain. So can another exact verifier similarly accept more tokens from the \emph{same tree}, or have existing exact verifiers already reached the tree's maximum achievable acceptance, leaving only execution efficiency to optimize? \emph{Second}, if existing verifiers already saturate the fixed-tree ceiling, how can target-model feedback improve future draft trees? Can we identify where target probability leaves the tree and use this signal to allocate the limited tree budget more effectively? These questions reveal two opportunities in tree speculative decoding: \emph{how much target probability the tree covers, and how efficiently that support is verified}.

\textbf{Our approach and contributions.}
We answer the both questions through a unified target-flow view of the realized draft tree.
First, we develop a verifier-independent theory showing that the maximum expected output-block length on a fixed tree is one plus the target probability mass covered by the selected support.
This theory exposes a canonical target-exit law, from which we derive \emph{Tree Exit Verification} (TEV) as a direct and implementation-friendly realization of optimal fixed-tree verification.
Finally, we use the same exit law as node-resolved training feedback to improve target-mass coverage under a finite tree budget, thereby increasing draft acceptance.
Fig.~\ref{fig:teaser} summarizes these three contributions.

\begin{figure}[t]\label{fig:teaser}
	\centering
	\includegraphics[width=\linewidth]{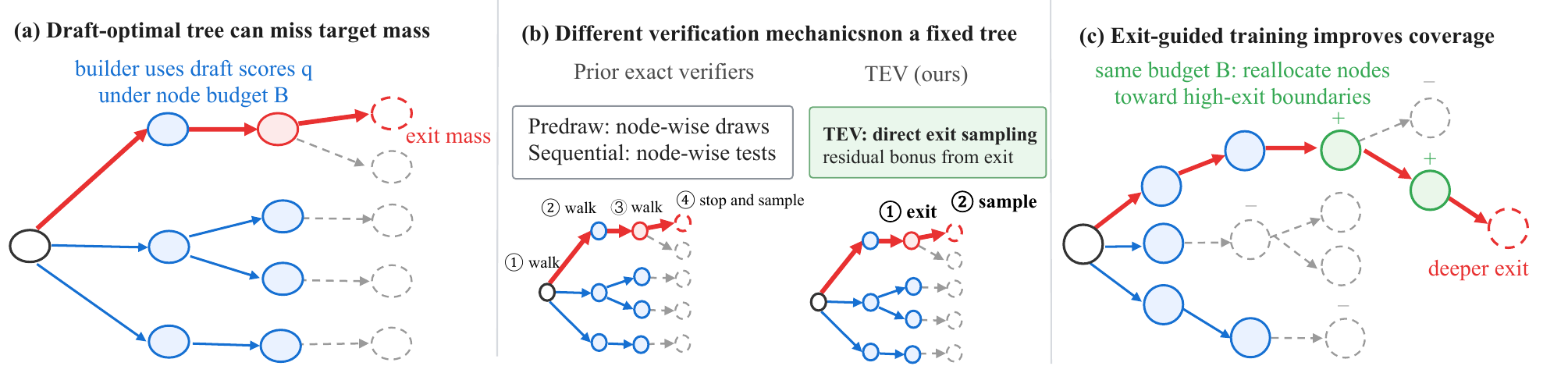}
	\vspace{-1.5mm}
	\vspace{-2.5mm}
	\caption{\textbf{Target-flow view of tree speculative decoding.} Blue denotes selected draft prefixes, gray dashed branches omitted candidates, red target-probability flow, and green reallocated tree budget. \textbf{(a)} Under a finite budget, draft-score-based tree construction can miss target-important branches, causing target mass to exit the tree. \textbf{(b)} On a fixed tree, prior exact verifiers realize the same optimal acceptance law through node-wise draws or tests, whereas TEV realizes it by direct exit-node sampling followed by bonus-token sampling. \textbf{(c)} Exit-guided training reallocates budget toward high-exit boundaries, improving target-mass coverage and output-block length.}
		\vspace{-4.5mm}
\end{figure}


First, we characterize the acceptance limit of exact path verifiers on a fixed tree: the maximum expected output-block length equals one plus its target-mass coverage, and every verifier attaining this bound induces the same canonical exit-node and bonus-token law. DDTree's predraw-and-follow~\citep{ringel2026ddtree} and EAGLE-style exact sequential residual verification~\citep{li2025eagle} realize this law and thus have identical accepted-path distributions on the same tree. Once verification is saturated, higher acceptance requires better tree coverage, not a different verifier.


Second, we derive \emph{Tree Exit Verification} (TEV), an exact, direct realization of this optimal law. Rather than using auxiliary per-node samples or sequential residual decisions, TEV samples an exit node directly and then a bonus token from its target residual (Fig.~\ref{fig:teaser}(b)). It preserves the fixed-tree acceptance ceiling while enabling sibling-order-invariant, level-parallel, fixed-shape execution: \emph{statistical saturation does not imply computational saturation}.


Third, we introduce \emph{Exit-guided Draft-Tree Training} (ExitTrain) to improve finite-budget tree coverage. ExitTrain evaluates the draft tree with the frozen target under actual tree-parent contexts and uses exit mass to reweight draft likelihoods, guiding budget allocation toward high-impact exit boundaries (Fig.~\ref{fig:teaser}(c)). Unlike token- or feature-level alignment~\citep{zhang2024learning,hu2025griffin,li2025eagle}, position-wise distillation~\citep{lei2026draftopd}, or scalar tree-reward optimization~\citep{hu2026gto}, this supervision localizes uncovered target mass directly on the draft tree.


Finally, we evaluate ExitTrain + TEV across dialogue, code, and mathematical reasoning benchmarks and multiple target-model families, against DFlash~\citep{chen2026dflash}, DDTree~\citep{ringel2026ddtree}, and Draft-OPD~\citep{lei2026draftopd}. Relative to DDTree with the same tree builder and budget, our system improves average output-block length (including the bonus token) by up to \textbf{13\%} and decode throughput by up to \textbf{14\%}. Separately, replacing predraw verification with TEV reduces verifier-stage latency by \textbf{15\%} without changing saturated expected output-block length on a fixed tree. These results support two complementary gains: \emph{better trees for higher acceptance and direct verification for lower latency}.

\vspace{-2mm}
\section{Background and Fixed-Tree Setup}
\label{sec:background}
\vspace{-2mm}

\paragraph{Tree speculative decoding.} 
A \emph{drafter} proposes candidates, a budget-$B$ \emph{builder} selects a prefix-closed tree, and one target pass lets a \emph{verifier} return an accepted path plus a bonus token. Systems combine multi-head, feature-autoregressive, recurrent, or block-diffusion drafting with parallel tree verification~\citep{cai2024medusa,li2024eagle,cheng2025redrafter,chen2026dflash,miao2024specinfer,chen2024sequoia,li2024eagle2,li2025eagle,svirschevski2024specexec}. Builders optimize this support using draft-side expected acceptance (OPT-Tree), block-diffusion marginals (DDTree), or autoregressive correction (DARTree)~\citep{wang2025opttree,ringel2026ddtree,li2026dartree}. We instantiate DFlash with DDTree's cumulative-score best-first builder~\citep{chen2026dflash,ringel2026ddtree}. These construction methods improve which tree is built, whereas our theory asks what any non-anticipating exact verifier can achieve once a prefix tree has been realized.

\begin{wrapfigure}{r}{0.4\textwidth}
	\vspace{-6mm}
	\centering
	\includegraphics[width=\linewidth]{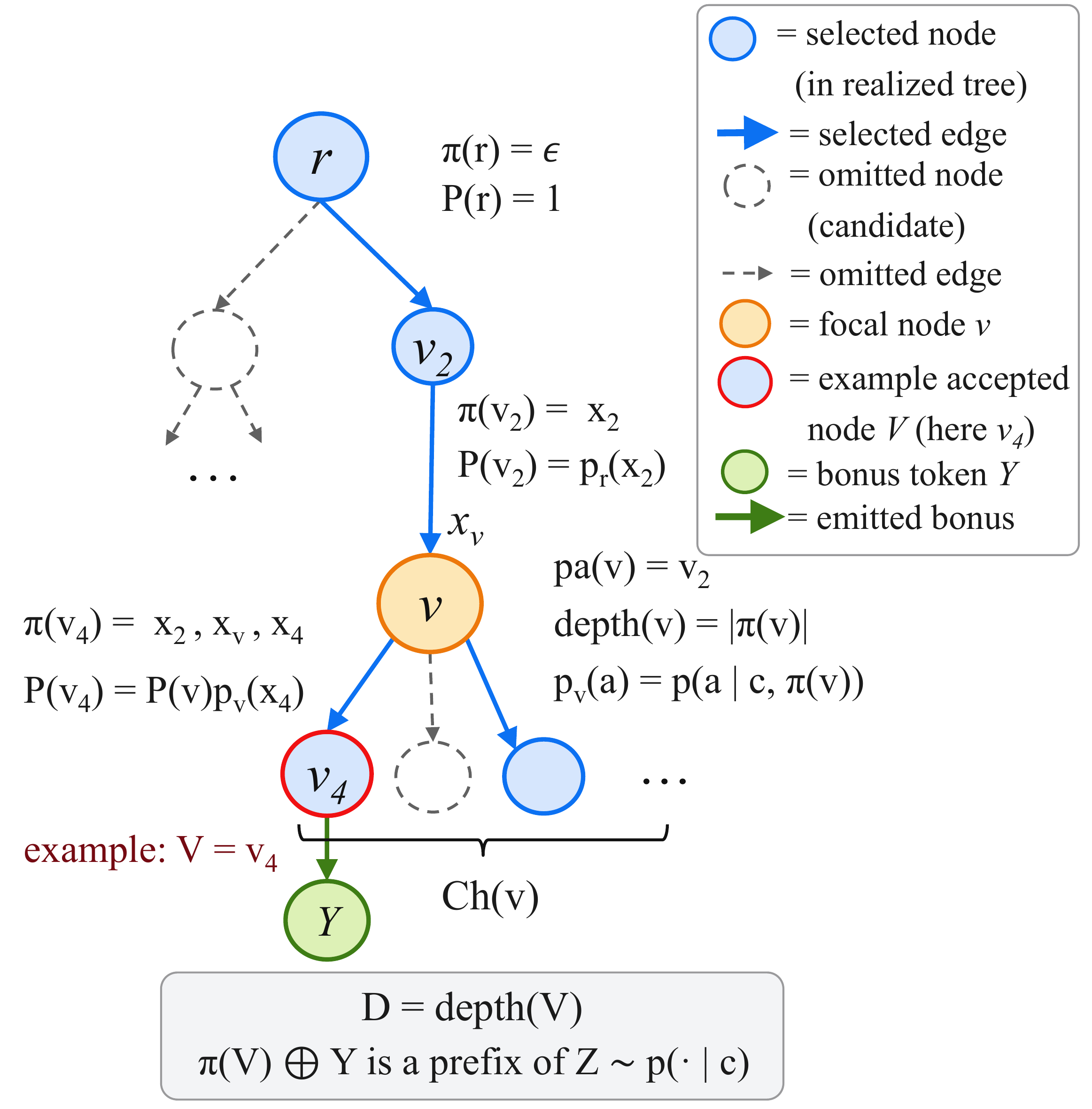}
	\vspace{-3mm}
	\caption{\textbf{Visualization of a fixed-tree.}}
	\label{fig:fixed_tree_setup}
	\vspace{-4mm}
\end{wrapfigure}

\paragraph{Two representative exact verifiers.}
DDTree's \emph{predraw-and-follow} independently samples one target token at each node~\citep{ringel2026ddtree}. Starting at the root, it follows a matching drafted child; otherwise, it stops and returns the sample as the bonus token. EAGLE-style \emph{exact sequential residual verification} instead tests children in order~\citep{li2025eagle}. After each rejection, it removes that child's target mass and renormalizes the residual, following the first accepted child or sampling a bonus token if all are rejected. Sec.~\ref{sec:coverage} shows that despite these different computations, both attain the same optimal law on a fixed tree. HSD~\citep{zhou2026hsd} instead verifies a sampled draft chain.

\paragraph{Related draft-model training.}
Prior work aligns drafters through on-policy or online distillation~\citep{zhou2024distillspec,liu2024online}, feature-, token-, or multi-step objectives~\citep{zhang2024learning,hu2025griffin,li2025eagle}, tree rewards~\citep{hu2026gto}, replayed trajectories~\citep{lei2026draftopd}, or acceptable-path modeling~\citep{zou2026variational}. Closest to ours, GTO optimizes a scalar tree reward~\citep{hu2026gto}, Learning to Draft learns throughput-driven depth and verification-size policies~\citep{zhang2026learning}, and VAT reweights simulated first-rejection positions~\citep{gu2026vat}. ExitTrain instead derives normalized node-level supervision from the canonical exit law of each realized finite tree, directly reallocating a fixed tree budget toward uncovered target mass.

\vspace{-2mm}
\section{The Canonical Exit Law and the Fixed-Tree Ceiling}
\label{sec:coverage}
\vspace{-2mm}

We condition on a decoding context and a realized draft tree. By following target probability through this fixed tree, we characterize both the largest expected output-block length any exact verifier can achieve and the unique law attained at that ceiling.

\paragraph{Fixed-tree setup.}
\label{sec:setup}
Let $\mathcal{X}$ be the token vocabulary and let $p(\cdot\mid c)$ denote the target continuation law given context $c$. A drafter and budget-$B$ builder produce a finite, rooted, prefix-closed tree $\mathcal{T}=(\mathcal{V},\mathcal{E},r)$ (Fig.~\ref{fig:fixed_tree_setup}), where $\mathcal{V}$ is the node set, $\mathcal{E}$ is the directed edge set, and $r$ is the root. Set $\pi(r)=\epsilon$, the empty sequence. Let $\mathcal{X}^{*}$ denote the set of finite token sequences. Each node $v$ represents a speculative continuation prefix $\pi(v)\in\mathcal{X}^{*}$; for $v\neq r$, let $\pa(v)$, $x_v$, $\depth(v)=|\pi(v)|$, and $\Ch(v)$ denote its parent, incoming token, depth, and children, so that $\pi(v)=\pi(\pa(v))\mathbin{\Vert}x_v$. Here $\Vert$ denotes sequence--token concatenation, $\alpha\preceq\beta$ means that sequence $\alpha$ is a prefix of $\beta$, and $\1[\cdot]$ denotes an indicator. Sibling tokens are distinct. Set $P(r)=1$; one target-model tree forward pass gives
\begin{align}
	p_v(a)&:=p(a\mid c,\pi(v)), \quad a\in\mathcal{X},
	\label{eq:node_target_distribution}\\
	P(v)&:=P(\pa(v))p_{\pa(v)}(x_v)=p(\pi(v)\mid c),
	\label{eq:prefix_probability}
\end{align}
where $P(v)$ is the target prefix probability. A verifier returns a path ending at $V$ and a bonus token $Y$. Let $D:=\depth(V)$ be the number of accepted draft tokens, and let $O:=\pi(V)\mathbin{\Vert}Y$ and $L:=|O|=D+1$ be the output block and its length, including the bonus token. Let $U$ denote auxiliary control randomness used, for example, by randomized early-stopping decisions and available before target tokens are revealed; randomness represented by the coupled target continuation $Z$ is not included in $U$. We call the verifier an \emph{exact path verifier} if there exists a coupling with a target continuation $Z=(Z_1,Z_2,\ldots)$ such that
\begin{equation}
    Z\mid(c,\mathcal{T},U)\sim p(\cdot\mid c),\qquad
    O=Z_{1:L}\ \text{a.s.},
    \label{eq:exact_path_coupling}
\end{equation}
and $L$ is a stopping time with respect to $\mathcal{H}_t:=\sigma(\mathcal{T},U,Z_1,\ldots,Z_t)$. Thus, the realized tree and independent control randomness are available at time zero, but neither may encode future target tokens. The target-prefix identity in Eq.~\eqref{eq:exact_path_coupling} is sufficient for the fixed-tree ceiling and equality characterization below; the stopping-time condition ensures that variable-length blocks compose safely across decoding rounds. Appendix~\ref{app:random_tree_conditioning} formalizes random construction and round-wise composition.

\vspace{-2mm}
\subsection{Target flow and the canonical exit law}
\label{sec:tev}
\label{sec:exit_law}
\vspace{-2mm}

We first describe how the target model interacts with a \emph{fixed} draft tree, without referring to any particular verifier. Consider a continuation $Z\sim p(\cdot\mid c)$ from the target model. Starting from the root, $Z$ follows matching drafted children for as long as its next token is represented in the tree. When the required next token is absent---or when a leaf is reached---the continuation leaves the draft tree. We call its deepest matched node the \emph{exit node}, denoted by $V^\star(Z)$ and illustrated in Fig.~\ref{fig:fixed_tree_setup}. Intuitively, the exit node marks \emph{where the draft tree stops covering a target continuation}.

For a node $v$, let
\begin{align}
	\mathcal{A}(v):=\{x_u:u\in\Ch(v)\}, \quad 
	\rho(v):=1-\sum\nolimits_{u\in\Ch(v)}p_v(x_u),\quad 
	w(v):=P(v)\rho(v).
	\label{eq:exit_weight}
\end{align}
Here, $\mathcal{A}(v)$ contains the draft tokens available after $v$. $P(v)$ is the probability that the target continuation reaches $v$. Conditioned on reaching $v$, the \emph{local exit probability} $\rho(v)$ is the fraction of target next-token probability not covered by its drafted children. So $w(v)$ is the overall probability that the target reaches $v$ and exits the draft tree there. For a leaf, no drafted child is available, so $\rho(v)=1$.

\begin{lemma}[Canonical exit law]
	\label{prop:exit_partition}
	The target-induced exit node satisfies
	\begin{equation}
		\Pr(V^\star(Z)=v\mid c,\mathcal{T})=w(v),
		\qquad
		\sum\nolimits_{v\in\mathcal{V}}w(v)=1.
		\label{eq:exit_partition}
	\end{equation}
\end{lemma}

The first identity in Lemma~\ref{prop:exit_partition} says that $w(v)$ is exactly the probability that a target continuation remains inside the draft tree up to $v$ and leaves there. The second follows because every target continuation has exactly one exit node in a finite tree. Thus, $\{w(v)\}_{v\in\mathcal{V}}$ is a normalized distribution over where the target leaves the draft tree, which we call the \emph{canonical exit law}. Importantly, this law is determined only by the target model and the fixed draft tree---not by a verifier. See the complete proof and probability-conservation argument in Appendix~\ref{app:exit_partition_proof}.

\subsection{From the exit law to the fixed-tree acceptance ceiling}
\label{sec:saturation}
\label{sec:canonical_saturation}

Let $D_{\mathsf{A}}$ and $L_{\mathsf{A}}=D_{\mathsf{A}}+1$ denote, respectively, the accepted draft-token count and output-block length of an exact path verifier $\mathsf{A}$. Accepting node $u$ requires the target continuation to contain prefix $\pi(u)$, so its acceptance probability is at most $P(u)$. Summing these node-wise limits motivates
\begin{equation}
	C_p(\mathcal{T})
	:=\sum\nolimits_{u\in\mathcal{V}\setminus\{r\}}P(u)
	=\sum\nolimits_{v\in\mathcal{V}}w(v)\depth(v),
	\label{eq:eal_two_decompositions}
\end{equation}
which we call the \emph{target coverage} of the draft tree.

The first form sums the target probability of every represented non-root prefix. The second, which follows from the exit law, identifies the same quantity as the expected exit depth, i.e., the expected number of accepted draft tokens. Because every round also emits one bonus token, the corresponding expected output-block length at saturation is $1+C_p(\mathcal{T})$.

For $\rho(v)>0$, define the residual target distribution
\begin{equation}
	p_v^{\mathrm{res}}(a)
	:={p_v(a)\1[a\notin\mathcal{A}(v)]}\ / \ {\rho(v)}.
	\label{eq:residual_distribution}
\end{equation}
which removes drafted-child tokens and renormalizes; at a leaf, $p_v^{\mathrm{res}}=p_v$. We call a verifier \emph{saturated} when $\mathbb{E}[D\mid c,\mathcal{T}]=C_p(\mathcal{T})$, equivalently when $\mathbb{E}[L\mid c,\mathcal{T}]=1+C_p(\mathcal{T})$. The following theorem gives this sharp ceiling and the unique output law at equality; Appendix~\ref{app:ceiling_uniqueness_proof} provides the proof.

\begin{theorem}[Fixed-tree ceiling and canonical saturation]
	\label{thm:retention_bound}
	\label{thm:unique_saturated_law}
	For any exact path verifier $\mathsf{A}$ and any non-root node $u$,
	\begin{equation}
		\Pr\nolimits_{\mathsf{A}}
		(u\text{ is accepted}\mid c,\mathcal{T})
		\leq P(u).
		\label{eq:retention_bound}
	\end{equation}
	Consequently,
	\begin{equation}
		\mathbb{E}[D_{\mathsf{A}}\mid c,\mathcal{T}]
		\leq C_p(\mathcal{T}),
		\qquad
		\mathbb{E}[L_{\mathsf{A}}\mid c,\mathcal{T}]
		\leq 1+C_p(\mathcal{T}).
		\label{eq:eal_upper_bound}
	\end{equation}
	If either length bound is tight, every node-wise bound is tight and, for every $v$ with $w(v)>0$,
	\begin{equation}
		\Pr\nolimits_{\mathsf{A}}(V=v\mid c,\mathcal{T})=w(v),
		\qquad
		Y\mid(V=v,c,\mathcal{T})\sim p_v^{\mathrm{res}}.
		\label{eq:canonical_output_law}
	\end{equation}
	Hence, all saturated exact path verifiers induce the same accepted-path and bonus-token law.
\end{theorem}

The theorem has two implications. First, $C_p(\mathcal{T})$ is the maximum expected number of accepted draft tokens, and $1+C_p(\mathcal{T})$ is the maximum expected output-block length reported by our experimental metric: no exact verifier can recover a drafted token absent from the tree. Second, saturation makes every node-wise bound tight, forcing the exit law $w$ and residual bonus-token law $p_v^{\mathrm{res}}$. Saturated verifiers may therefore differ computationally but not in their accepted-path distribution. The representative verifiers in Sec.~\ref{sec:background} attain this ceiling (Appendix~\ref{app:verifier_equivalence_proof}).

\begin{corollary}[Representative verifiers are saturated]
	\label{thm:verifier_equivalence}
	Predraw-and-follow in DDTree~\citep{ringel2026ddtree} and exact sequential residual verification in EAGLE-style tree decoding~\citep{li2025eagle} both attain the fixed-tree ceiling.
	Therefore, on the same draft tree and target distributions, they share the canonical exit law, accepted-path distribution, expected accepted draft-token count, and expected output-block length.
\end{corollary}

Whenever node $v$ is reached, both representative verifiers follow each drafted child $u$ with probability $p_v(x_u)$ and exit with the remaining probability $\rho(v)$. They therefore realize the same canonical law despite different internal computations.

\subsection{What remains after fixed-tree saturation}
\label{sec:fixed_tree_implications}
The fixed-tree result leaves two opportunities. The same law can be executed more efficiently, motivating TEV in Sec.~\ref{sec:direct_sampler}: \emph{statistical saturation does not imply computational saturation}. Higher output-block length, however, requires increasing $C_p(\mathcal{T})$ itself. A large $w(v)$ localizes target mass that reaches $v$ but leaves the current tree, motivating the training method in Sec.~\ref{sec:draft_training}. This result complements HSD~\citep{zhou2026hsd}, which exploits target-to-draft ratios while verifying a random sampled chain rather than conditioning on a realized tree (Appendix~\ref{app:sampled_chain_relation}).

\vspace{-2mm}
\section{TEV: Direct Verification from the Canonical Exit Law}
\label{sec:direct_sampler}
\vspace{-2mm}

Sec.~\ref{sec:coverage} resolves the statistical question: saturated exact verifiers share the same canonical exit law and output-block length on a fixed tree. The remaining freedom is computational---\emph{how can this common law be realized efficiently?} Existing verifiers recover it through per-node draws or sequential child decisions. Tree Exit Verification (TEV) instead realizes it directly.

Theorem~\ref{thm:unique_saturated_law} specifies the required output: exit at $v$ with probability $w(v)$ and then draw the bonus token from $p_v^{\mathrm{res}}$. TEV therefore computes the exit distribution, draws one exit node, and samples one residual bonus token. Algorithm~\ref{alg:tev} in Appendix~\ref{app:tev} gives the complete procedure after the shared target-model tree forward pass produces $\{p_v\}$. The drafter need not be queried again: it has already determined what appears in the realized tree, while the target distributions determine how probability flows through that support.

Because Algorithm~\ref{alg:tev} implements the canonical law, its correctness follows from Sec.~\ref{sec:saturation}; Appendix~\ref{app:tev_exactness_proof} gives the proof.

\begin{corollary}[\TEV{} is exact and saturated]
	\label{prop:tev_retention}
	For any context and realized prefix-closed draft tree, \TEV{} is an exact path verifier and attains the fixed-tree ceiling.
	In particular, we have 
	\begin{equation}
		\Pr\nolimits_{\TEV}(u\text{ is accepted}\mid c,\mathcal{T})
		=P(u)  \quad (\forall u\in\mathcal{V}\setminus\{r\}),
		\qquad
		\mathbb{E}[L_{\TEV}\mid c,\mathcal{T}]
		=1+C_p(\mathcal{T}).
		\label{eq:tev_node_retention}
	\end{equation}
\end{corollary}

By Lemma~\ref{prop:exit_partition}, a target continuation exits according to $w$ and produces its next token according to $p_v^{\mathrm{res}}$. TEV reproduces these distributions; node $u$ is accepted exactly when the sampled exit lies in its subtree, whose total exit mass is $P(u)$.

\textbf{Efficiency}. TEV shares the accepted-path law of predraw and sequential residual verification, but removes their per-node samples or data-dependent child tests. It uses node-local reductions, depth-wise propagation, and one exit-node and bonus-token draw, yielding a sibling-order-invariant, fixed-shape, level-parallel procedure. Sec.~\ref{sec:ablation} evaluates this latency benefit separately from output-block length.

\vspace{-2mm}
 \section{ExitTrain: Exit-Guided Draft-Tree Training}
 \label{sec:draft_training}
 \vspace{-2mm}
 
Sec.~\ref{sec:direct_sampler} addresses the remaining computational freedom: TEV changes how the canonical law is executed, not the fixed-tree acceptance ceiling. We now turn to the second question from Sec.~\ref{sec:introduction}: \emph{how can target feedback build a better draft tree?} Theorem~\ref{thm:unique_saturated_law} provides the key: once verification is saturated, higher output-block length requires greater target coverage, and the exit law reveals where that coverage is missing.

Let $H$ be the maximum draft depth. For an expandable node $v$ with $\depth(v)<H$ and a missing child token $a\notin\mathcal{A}(v)$, write $\mathcal{T}\oplus(v,a)$ for the tree obtained by adding child $(v,a)$. Then
\begin{align}
 C_p(\mathcal{T}\oplus(v,a))-C_p(\mathcal{T})&=P(v)p_v(a),
 \label{eq:one_step_marginal_gain}\\
 \sum\nolimits_{a\notin\mathcal{A}(v)}P(v)p_v(a)&=P(v)\rho(v)=w(v).
 \label{eq:exit_as_expansion_gain}
\end{align}
Hence, for an expandable node $v$, $w(v)$ measures the total one-step target coverage currently missing below $v$. A large $w(v)$ means that the target frequently reaches this prefix but substantial next-token probability is not represented by its drafted children. These are precisely the regions where allocating more tree capacity can keep target continuations inside the tree for longer. This expansion interpretation does not apply to a hard-horizon leaf. If $H$ is the maximum draft depth and $\depth(v)=H$, then $\rho(v)=1$ and $w(v)=P(v)$; weighting its incoming token rewards a path that already carries target mass to the full horizon rather than missing one-step coverage.

This observation directly suggests our training strategy. Let $q_\theta$ be the current drafter with parameters $\theta$. For each training context $c$, $q_\theta$ and the same tree builder used at inference construct a bounded tree $\widehat{\mathcal{T}}_\theta(c)$. We keep its structure fixed within the update, evaluate every node with the frozen target model under its actual tree-parent context, and compute the exit probabilities $w(v)$. Training therefore receives node-level feedback about where the \emph{current inference-time draft tree} fails to cover target probability, rather than supervision from only a single draft path.

For a minibatch $\mathcal{B}$ of contexts $\{c_i\}_{i\in\mathcal{B}}$, let $\widehat{\mathcal{T}}_i=(\widehat{\mathcal{V}}_i,\widehat{\mathcal{E}}_i,r)$ be the realized tree for example $i$, let $w_i(v)$ be its detached target-derived exit weight, and let $q_{\theta,d}(\cdot\mid c_i)$ be DFlash's position-wise marginal distribution at future draft position $d$. We reweight the likelihood of the incoming token at each non-root node:
 \begin{equation}
 	\mathcal{L}_{\mathrm{exit}}(\theta)
 	:=
 	-
 	\frac{
		\sum_{i\in\mathcal{B}}
		\sum_{v\in\widehat{\mathcal{V}}_i\setminus\{r\}}
		w_i(v)\log q_{\theta,\depth(v)}(x_v\mid c_i)}
 	{\sum_{i\in\mathcal{B}}
 		(|\widehat{\mathcal{V}}_i|-1)}.
 	\label{eq:tree_auxiliary_loss}
 \end{equation}
Consequently, nodes at the same depth share the same draft row even when they have different tree-parent prefixes; in contrast, $p_v$ is evaluated under the actual prefix $(c_i,\pi(v))$. All target-derived quantities and the realized tree structure are detached; gradients flow only through the drafter.
 
Why reinforce the position-wise marginal probability of the incoming draft token of an exit node? In DDTree, a prefix is ranked by cumulative draft scores along its path. Increasing the marginal probability assigned to $x_v$ at position $\depth(v)$ therefore raises the future priority of $v$ and its descendants, making the builder more likely to retain and further expand that region. Weighting this update by $w(v)$ directs future tree budget toward prefixes that the target frequently reaches but the current tree insufficiently expands. Because the top-$B$ builder is discrete, this is a \emph{tree-allocation surrogate}, not the exact gradient of target coverage; its role is to improve \emph{where} the bounded tree spends its nodes.
 
 Finally, let $\mathcal{L}_{\mathrm{tok}}$ denote the standard dense token-prediction loss and let $\lambda_{\mathrm{exit}}\geq0$ control the exit-guided term. We optimize
 \begin{equation}
 	\mathcal{L}(\theta)
 	=
 	\mathcal{L}_{\mathrm{tok}}(\theta)
 	+
 	\lambda_{\mathrm{exit}}\mathcal{L}_{\mathrm{exit}}(\theta).
 	\label{eq:hybrid_loss}
 \end{equation} The two terms are complementary: the dense token loss teaches the drafter \emph{what to predict}, while the exit-guided loss teaches a bounded tree \emph{where to allocate its capacity}. Our default loss includes hard-horizon leaves, so $w(v)$ has two roles: missing-coverage credit at expandable nodes and successful-deep-reach credit at depth $H$. The root requires separate treatment: $w(r)$ measures target mass missed by every first-position draft token, but $r$ has no incoming draft edge and is omitted from $\mathcal{L}_{\mathrm{exit}}$. Consequently, the included exit weights sum to $1-w(r)$ rather than one, and the root-exit mass receives no direct credit from $\mathcal{L}_{\mathrm{exit}}$; coverage of missing first-position tokens is addressed through the dense token objective. In this way, the same canonical exit law serves two distinct roles: TEV uses it to realize fixed-tree verification directly, while exit-guided training uses it to improve the tree itself. Sec.~\ref{sec:ablation} studies alternative weights, and Appendix~\ref{app:efficient_loss} gives an efficient implementation.

\vspace{-2mm}
\section{Experiments}
\label{sec:experiments}
\vspace{-2mm}

\textbf{Setup.} We evaluate Qwen3-4B/8B~\citep{yang2025qwen3}, LLaMA-3.1-8B~\citep{dubey2024llama}, Gemma-4-12B-IT~\citep{gemma2026modelcard}, and Qwen3-Coder-30B-A3B~\citep{yang2025qwen3} on MT-Bench~\citep{zheng2023judging}, HumanEval~\citep{chen2021evaluating}, GSM8K~\citep{cobbe2021training}, and MATH-500~\citep{hendrycks2021measuring}. Drafter training uses \texttt{open-perfectblend}~\citep{xu2024perfectblendredefiningrlhf} prompts and online responses rolled out by each corresponding target model (Appendix~\ref{app:experimental_details}).

\textbf{Methods.} DDTree uses the released DFlash drafter, cumulative-score best-first builder, and predraw verifier~\citep{chen2026dflash,ringel2026ddtree}; Draft-OPD replaces only the drafter with its position-wise distilled checkpoint~\citep{lei2026draftopd}. \textbf{ExitTrain + TEV} denotes the complete system combining our exit-guided drafter with TEV. Appendix~\ref{app:experimental_details} gives full details.

\textbf{Metrics.} Because all methods preserve the target distribution, we report speedup over autoregressive decoding ($SR$) and average output-block length per draft--verify cycle ($\tau=\mathbb{E}[L]$), including the bonus token, at $T\in\{0,1\}$.

\begin{table}[t]
    \caption{Speedup ratio ($SR$) and average output-block length including the bonus token ($\tau$) on MT-Bench, HumanEval, GSM8K, and MATH-500. All tree-based systems share the same builder and tree budget $B=64$. ExitTrain + TEV combines our exit-guided drafter with TEV.}
    \label{tab:main_results}
    \vskip 0.08in
    \vspace{-4mm}
    \begin{center}
    \begin{small}
    \resizebox{1.00\columnwidth}{!}{
    \setlength{\tabcolsep}{4pt}
    \begin{tabular}{ll||cc|cc|cc|cc||cc}
        \toprule
        & & \multicolumn{2}{c|}{\textbf{MT-Bench}} & \multicolumn{2}{c|}{\textbf{HumanEval}} & \multicolumn{2}{c|}{\textbf{GSM8K}} & \multicolumn{2}{c||}{\textbf{MATH-500}} & \multicolumn{2}{c}{\textbf{Average}} \\
        \cmidrule{3-12}
        \textbf{Model} & \textbf{Method} & $SR\uparrow$ & $\tau\uparrow$ & $SR\uparrow$ & $\tau\uparrow$ & $SR\uparrow$ & $\tau\uparrow$ & $SR\uparrow$ & $\tau\uparrow$ & $SR\uparrow$ & $\tau\uparrow$ \\
        \midrule
        \multicolumn{12}{c}{\textit{Temperature = 0}} \\
        \midrule
        \multirow{3}{*}{Qwen3-4B}
        & DDTree & 3.81 & 5.41 & 5.71 & 8.13 & 5.65 & 8.48 & 6.83 & 9.16 & 5.50 & 7.80 \\
        & Draft-OPD & 4.16 & 5.73 & 5.96 & 8.35 & 6.11 & 9.30 & 7.02 & 10.14 & 5.81 & 8.38 \\
        & \textbf{ExitTrain + TEV} & \textbf{4.28} & \textbf{6.17} & \textbf{6.09} & \textbf{8.74} & \textbf{6.91} & \textbf{10.22} & \textbf{7.23} & \textbf{10.26} & \textbf{6.13} & \textbf{8.85} \\
        \cmidrule{1-12}
        \multirow{3}{*}{Qwen3-8B}
        & DDTree & 3.75 & 5.54 & 5.83 & 8.01 & 5.88 & 8.49 & 6.23 & 9.76 & 5.43 & 7.95 \\
        & Draft-OPD & 3.90 & 5.92 & 6.02 & 8.46 & 6.27 & 9.33 & 6.84 & 10.17 & 5.76 & 8.47 \\
        & \textbf{ExitTrain + TEV} & \textbf{4.22} & \textbf{6.12} & \textbf{6.15} & \textbf{8.73} & \textbf{7.20} & \textbf{10.39} & \textbf{7.15} & \textbf{10.69} & \textbf{6.18} & \textbf{8.98} \\
        \cmidrule{1-12}
        \multirow{2}{*}{LLaMA-3.1-8B}
        & DDTree & 3.55 & 5.53 & 4.33 & 6.53 & 3.79 & 5.89 & 3.70 & 5.71 & 3.84 & 5.92 \\
        & \textbf{ExitTrain + TEV} & \textbf{3.69} & \textbf{5.82} & \textbf{4.72} & \textbf{7.24} & \textbf{4.75} & \textbf{7.48} & \textbf{4.70} & \textbf{7.34} & \textbf{4.46} & \textbf{6.97}  \\
        \cmidrule{1-12}
        \multirow{2}{*}{Gemma-4-12B-IT}
        & DDTree & 4.17 & 5.44 & 4.83 & 5.66 & 4.39 & 5.29 & 4.87 & 5.86 & 4.57 & 5.56 \\
        & \textbf{ExitTrain + TEV} & \textbf{4.35} & \textbf{5.57} & \textbf{5.09} & \textbf{5.92} & \textbf{4.70} & \textbf{5.66} & \textbf{5.17} & \textbf{6.33} & \textbf{4.83} & \textbf{5.87} \\
        \cmidrule{1-12}
        \multirow{2}{*}{Qwen3-Coder-30B-A3B}
        & DDTree & 3.70 & 4.87 & 7.57 & 9.95 & 5.26 & 6.81 & 5.61 & 7.38 & 5.53 & 7.25 \\
        & \textbf{ExitTrain + TEV} & \textbf{3.83} & \textbf{5.03} & \textbf{7.83} & \textbf{10.32} & \textbf{5.61} & \textbf{7.43} & \textbf{6.00} & \textbf{7.74} & \textbf{5.82} & \textbf{7.63}  \\
        \midrule
        \multicolumn{12}{c}{\textit{Temperature = 1}} \\
        \midrule
        \multirow{3}{*}{Qwen3-4B}
        & DDTree & 3.44 & 5.04 & 5.22 & 7.53 & 5.48 & 7.85 & 5.65 & 7.88 & 4.95 & 7.08 \\
        & Draft-OPD & 3.74 & 5.26 & 5.49 & 7.80 & 5.88 & 8.58 & 6.08 & 8.50 & 5.30 & 7.54 \\
        & \textbf{ExitTrain + TEV} & \textbf{3.98} & \textbf{5.86} & \textbf{5.64} & \textbf{8.06} & \textbf{6.47} & \textbf{9.25} & \textbf{6.26} & \textbf{9.08} & \textbf{5.59} & \textbf{8.06} \\
        \cmidrule{1-12}
        \multirow{3}{*}{Qwen3-8B}
        & DDTree & 3.42 & 5.03 & 4.83 & 7.06 & 5.23 & 7.84 & 5.50 & 8.17 & 4.74 & 7.03 \\
        & Draft-OPD & 3.51 & 5.21 & 4.95 & 7.29 & 5.79 & 8.57 & 5.66 & 8.28 & 4.98 & 7.34 \\
        & \textbf{ExitTrain + TEV} & \textbf{3.74} & \textbf{5.40} & \textbf{5.11} & \textbf{7.47} & \textbf{6.35} & \textbf{9.12} & \textbf{5.96} & \textbf{8.59} & \textbf{5.29} & \textbf{7.65} \\
        \cmidrule{1-12}
        \multirow{2}{*}{LLaMA-3.1-8B}
        & DDTree & 2.43 & 3.97 & 3.71 & 5.67 & 3.07 & 4.81 & 2.26 & 3.71 & 2.87 & 4.54 \\
        & \textbf{ExitTrain + TEV} & \textbf{2.52} & \textbf{4.08} & \textbf{3.84} & \textbf{5.93} & \textbf{3.67} & \textbf{5.77} & \textbf{2.63} & \textbf{4.25} & \textbf{3.16} & \textbf{5.01} \\
        \cmidrule{1-12}
        \multirow{2}{*}{Gemma-4-12B-IT}
        & DDTree & 3.65 & 4.71 & 4.13 & 5.06 & 4.48 & 5.41 & 4.57 & 5.78 & 4.21 & 5.24 \\
        & \textbf{ExitTrain + TEV} & \textbf{3.83} & \textbf{4.83} & \textbf{4.39} & \textbf{5.28} & \textbf{4.70} & \textbf{5.83} & \textbf{4.74} & \textbf{6.02} & \textbf{4.41} & \textbf{5.49} \\
        \cmidrule{1-12}
        \multirow{2}{*}{Qwen3-Coder-30B-A3B}
        & DDTree & 3.48 & 4.53 & 7.26 & 9.61 & 5.22 & 6.75 & 5.43 & 7.13 & 5.35 & 7.01 \\
        & \textbf{ExitTrain + TEV} & \textbf{3.61} & \textbf{4.72} & \textbf{7.39} & \textbf{9.92} & \textbf{5.57} & \textbf{7.17} & \textbf{5.78} & \textbf{7.42} & \textbf{5.59} & \textbf{7.31} \\
        \bottomrule
    \end{tabular}}
    \end{small}
    \end{center}
    \vskip -0.10in
    \vspace{-4mm}
\end{table}

\subsection{Main Results}
\label{sec:main_results}

Table~\ref{tab:main_results} shows that ExitTrain + TEV achieves the highest $\tau$ and $SR$ in every reported setting. The Qwen3-4B/8B comparisons provide the most controlled view because they use the same builder and budget and include both DDTree and Draft-OPD. Across the two model sizes and temperatures, ExitTrain + TEV improves average $\tau$ over DDTree by 8.8\%--13.8\% and average $SR$ by 11.5\%--13.8\%. It also exceeds Draft-OPD by 4.2\%--6.9\% in $\tau$ and 5.5\%--7.3\% in $SR$, showing that the gains persist over both the released and position-wise-distilled drafters under greedy and stochastic decoding.

The improvement also transfers across target-model families. On LLaMA-3.1-8B, ExitTrain + TEV improves average $\tau$ by 17.7\%/10.4\% and $SR$ by 16.1\%/10.1\% at $T{=}0/1$. The corresponding $\tau$ and $SR$ gains are 5.6\%/4.8\% and 5.7\%/4.8\% on Gemma-4-12B-IT, and 5.2\%/4.3\% and 5.2\%/4.5\% on Qwen3-Coder-30B-A3B. Thus, the advantage remains positive for dense, instruction-tuned, and MoE targets, although its magnitude varies across model families.

The gains also exhibit meaningful task variation. On Qwen3, GSM8K yields the largest increases in $\tau$: 20.5\%--22.4\% at $T{=}0$ and 16.3\%--17.8\% at $T{=}1$, while HumanEval shows smaller but consistent improvements. More broadly, increases in $\tau$ translate consistently into higher $SR$. Because the builder and budget ($B=64$) are fixed, the additional accepted length cannot come from verifying more candidates; it is the empirical signature of better target-prefix coverage, which raises the saturated output-block length $1+C_p(\mathcal{T})$.

\paragraph{Crossing drafter training and verification.}
The crossed evaluation in Appendix~\ref{app:crossed_experiment} confirms this attribution on Qwen3-4B at $T{=}1$. Holding Predraw fixed, exit training raises average $\tau$ from 7.08 to 8.05 and $SR$ from $4.95\times$ to $5.55\times$; with TEV, the corresponding changes are 7.09 to 8.06 and $4.97\times$ to $5.59\times$. In contrast, for a fixed drafter, Predraw and TEV differ by at most 0.03 tokens per benchmark, consistent with equal expected output-block length under Theorem~\ref{thm:unique_saturated_law}; the small empirical discrepancies arise from separate finite-sample decoding runs. TEV nevertheless improves average end-to-end $SR$ by 0.4\% with the original drafter and 0.7\% with the exit-trained drafter. Together, these results separate the two contributions: ExitTrain supplies most of the acceptance gain by improving the tree, while TEV provides orthogonal execution savings, reducing verifier-decision and cache-commit latency by 12.0\%--24.3\% across the tested budgets (Appendix~\ref{app:verifier_efficiency}).

\subsection{Ablation Study}
\label{sec:ablation}

\begin{table}[t]
    \caption{Effect of the draft-training objective on Qwen3-4B. All three trained drafters share the initialization, the \texttt{open-perfectblend} mixture, the compute budget, the tree builder, and the budget $B=64$}
    \label{tab:training_ablation}
    \vskip 0.08in
    \vspace{-4mm}
    \begin{center}
    \begin{small}
    \resizebox{1.00\columnwidth}{!}{
    \setlength{\tabcolsep}{4pt}
    \begin{tabular}{l||cc|cc|cc|cc||cc}
        \toprule
        & \multicolumn{2}{c|}{\textbf{MT-Bench}} & \multicolumn{2}{c|}{\textbf{HumanEval}} & \multicolumn{2}{c|}{\textbf{GSM8K}} & \multicolumn{2}{c||}{\textbf{MATH-500}} & \multicolumn{2}{c}{\textbf{Average}} \\
        \cmidrule{2-11}
        \textbf{Draft training objective} & $SR\uparrow$ & $\tau\uparrow$ & $SR\uparrow$ & $\tau\uparrow$ & $SR\uparrow$ & $\tau\uparrow$ & $SR\uparrow$ & $\tau\uparrow$ & $SR\uparrow$ & $\tau\uparrow$ \\
        \midrule
        \multicolumn{11}{c}{\textit{Temperature = 0}} \\
        \midrule
        Original DFlash drafter & 3.81 & 5.41 & 5.71 & 8.13 & 5.65 & 8.48 & 6.83 & 9.16 & 5.50 & 7.80 \\
        Finetuned DFlash drafter & 3.97 & 5.71 & 5.83 & 8.36 & 6.33 & 9.54 & 6.91 & 9.83 & 5.76 & 8.36 \\
        Prefix weight $P(v)$ & 4.05 & 5.82 & 5.99 & 8.43 & 6.07 & 9.25 & 6.71 & 9.33 & 5.71 & 8.21 \\
        \textbf{Exit weight $w(v)$ (ours)} & \textbf{4.28} & \textbf{6.17} & \textbf{6.09} & \textbf{8.74} & \textbf{6.91} & \textbf{10.22} & \textbf{7.23} & \textbf{10.26} & \textbf{6.13} & \textbf{8.85} \\
        \midrule
        \multicolumn{11}{c}{\textit{Temperature = 1}} \\
        \midrule
        Original DFlash drafter & 3.44 & 5.04 & 5.22 & 7.53 & 5.48 & 7.85 & 5.65 & 7.88 & 4.95 & 7.08 \\
        Finetuned DFlash drafter & 3.58 & 5.20 & 5.43 & 7.63 & 6.18 & 8.58 & 5.86 & 8.46 & 5.26 & 7.47 \\
        Prefix weight $P(v)$ & 3.69 & 5.45 & 5.48 & 7.88 & 6.26 & 8.85 & 5.99 & 8.68 & 5.35 & 7.72 \\
        \textbf{Exit weight $w(v)$ (ours)} & \textbf{3.98} & \textbf{5.86} & \textbf{5.64} & \textbf{8.06} & \textbf{6.47} & \textbf{9.25} & \textbf{6.26} & \textbf{9.08} & \textbf{5.59} & \textbf{8.06} \\
        \bottomrule
    \end{tabular}}
    \end{small}
    \end{center}
    \vskip -0.10in
    \vspace{-4mm}
\end{table}

\paragraph{Exit-guided training.}
We compare the released DFlash checkpoint with three variants trained under matched initialization, data, and compute: token-only finetuning, prefix-weighted training, and exit-guided training. All use the same builder and $B=64$. The prefix-weighted control replaces $w(v)$ with $P(v)$, rescaled to match the total non-root exit weight while retaining the same loss denominator. The two controls therefore separate the benefit of additional training from the benefit of assigning credit specifically to exit boundaries.

Table~\ref{tab:training_ablation} first shows that additional training is itself useful: token-only finetuning increases average $\tau$ from 7.80/7.08 to 8.36/7.47 at $T{=}0/1$. Exit-guided training then improves over this matched control by a further 5.9\%/7.9\% in $\tau$ and 6.4\%/6.3\% in $SR$, demonstrating that its gain is not explained by additional training on \texttt{open-perfectblend} alone. The $SR$ numbers compare the complete systems; the crossed evaluation in Sec.~\ref{sec:main_results} shows that TEV contributes only 0.7\% end-to-end with the exit-trained drafter, so the dominant improvement comes from training.

The prefix-weighted control tests whether target reachability alone provides sufficient tree supervision. It reaches $\tau=8.21/7.72$, whereas exit weighting reaches 8.85/8.06, improvements of 7.8\%/4.4\%; the corresponding $SR$ gains are 7.4\%/4.5\%. Moreover, prefix weighting trails token-only finetuning at $T{=}0$, despite using target prefix probabilities. The distinction is that $P(v)$ measures how often a node is reached, whereas $w(v)=P(v)\rho(v)$ additionally identifies reached boundaries whose selected children leave substantial target mass uncovered (Eq.~\eqref{eq:exit_as_expansion_gain}). The results therefore support localized exit feedback rather than reachability weighting alone.

\paragraph{Verifier-stage efficiency.}
Appendix~\ref{app:verifier_efficiency} and Table~\ref{tab:verifier_latency} isolate verifier execution by holding the drafter, builder, and realized tree fixed and excluding the shared target-model tree forward pass. TEV reduces verifier-decision and cache-commit latency at every tested budget: 12.9\%/12.0\% at $B=32$, 15.2\%/14.0\% at $B=64$, and 24.3\%/17.1\% at $B=128$ for $T=0.6/1.0$. The increasing advantage at wider trees is consistent with replacing Predraw's per-node sampling with level-wise probability propagation and one exit-node draw. Because this stage occupies only a small fraction of each decoding cycle, these reductions translate into the smaller end-to-end gains observed in the crossed evaluation.

\paragraph{Tree-budget robustness.}
Appendix~\ref{app:budget_ablation} and Table~\ref{tab:budget_ablation} evaluate the $B=64$-trained drafter at $B\in\{32,64,128\}$ without budget-specific retraining. ExitTrain + TEV remains best in all six budget--temperature settings: relative to DDTree, it improves $\tau$ by 11.6\%--13.8\% and $SR$ by 11.4\%--12.9\%; relative to Draft-OPD, the gains are 5.6\%--6.9\% and 4.5\%--6.2\%. Its advantage therefore persists for both narrower and wider inference trees, indicating that the learned exit-guided allocation is not tied to the training-time tree size.

\vspace{-2mm}
\section{Conclusion}
\label{sec:conclusion}
\vspace{-2mm}

We study tree speculative decoding through target probability flow and show that the maximum expected output-block length on a fixed tree is one plus its target-mass coverage. The canonical exit law gives a verifier-independent ceiling: saturating exact verifiers share the same exit and bonus-token law, so missing support cannot be recovered by verification alone. This suggests two complementary directions. TEV realizes the optimal law with a regular, level-parallel procedure to reduce verification cost, while exit-guided training uses node-level exit signals to improve finite-budget coverage and output-block length. Overall, our framework separates \emph{what tree to build} from \emph{how to verify it efficiently}.

\paragraph{Limitations.}
ExitTrain incurs one-time training overhead from target rollouts and tree evaluation, but fine-tunes released DFlash checkpoints without architectural or verifier changes and can be amortized during deployment. Our systems results are limited to batch-one decoding with DDTree; alternative builders and high-throughput serving remain future work.

\section*{AI Use Statement}
Generative AI tools were used solely for language polishing and improving the clarity and readability of the manuscript. The authors reviewed all AI-assisted edits and take full responsibility for the final content of this work.

\section*{Ethics Statement}
This work studies the efficiency of lossless large language model inference. It does not involve human subjects, private user data, or the collection of new personal information. Our experiments use publicly available models and benchmarks under their respective terms of use. By construction, TEV preserves the output distribution of the target model and does not intentionally modify its capabilities, safety behavior, or intended use. More efficient inference may reduce computational and energy costs, but may also lower the cost of deploying language models at scale. Consequently, downstream uses of our method inherit the potential safety, bias, privacy, and misuse risks of the underlying target models. Users remain responsible for evaluating these risks and complying with the licenses and usage policies of the deployed models and datasets.

\section*{Reproducibility Statement}
Sections~\ref{sec:coverage}--\ref{sec:draft_training} provide the definitions, assumptions, fixed-tree results, TEV construction, and exit-guided training objective used in this work. Complete proofs of the fixed-tree ceiling, verifier equivalence, and TEV exactness are provided in Appendices~\ref{app:fixed_tree_theory} and~\ref{app:tev}. Additional derivations and implementation details for the exit-guided objective are given in Appendix~\ref{app:exit_training_details}, while the model, dataset, training, hardware, and evaluation configurations are documented in Appendix~\ref{app:experimental_details}. An anonymized repository at \url{https://github.com/hsj576/TEV} will provide the training and inference code, configuration files, evaluation scripts, and instructions required to reproduce the reported experiments.

\bibliography{iclr2027_conference}
\bibliographystyle{iclr2027_conference}

\appendix

\section{Detailed Theory for Fixed-Tree Verification}
\label{app:fixed_tree_theory}

This appendix provides the technical details for Section~\ref{sec:coverage}.
Unless stated otherwise, we condition on a decoding context $c$ and a fixed,
finite, rooted, prefix-closed draft tree
$\mathcal{T}=(\mathcal{V},\mathcal{E},r)$.
Sibling edges carry distinct draft tokens.
For each node $v$, $p_v$ is its target next-token distribution, $P(v)$ is its
target prefix probability, $\rho(v)$ is its local exit probability, and
$w(v)=P(v)\rho(v)$ is its exit probability.

\subsection{Random trees, auxiliary randomness, and round-wise composition}
\label{app:random_tree_conditioning}

The results in Section~\ref{sec:coverage} are stated conditional on a realized
draft tree, so the tree builder may itself be deterministic or stochastic.
Let $Q(\cdot\mid c)$ denote the conditional law of $\mathcal{T}$, which is
constructed before verification, and let $R(\cdot\mid c,\mathcal{T})$ denote
the conditional law of auxiliary control randomness $U$ available at time zero.
This $U$ may randomize stopping or control decisions but excludes randomness
represented by the coupled target continuation.  For each realized pair
$(\mathcal{T},U)$,
suppose there exists a measurable conditional coupling
$\Gamma_{\mathcal{T},U}$ of the output block $O$, a target continuation $Z$,
and its length $L=|O|$ such that
\begin{equation}
    Z\mid(c,\mathcal{T},U)\sim p(\cdot\mid c),
    \qquad
    O=Z_{1:L}\quad \Gamma_{\mathcal{T},U}\text{-almost surely},
    \label{eq:app_conditional_coupling}
\end{equation}
and $L$ is a stopping time with respect to
$\mathcal{H}_t=\sigma(\mathcal{T},U,Z_1,\ldots,Z_t)$.  Construct the joint
coupling by first drawing $\mathcal{T}\sim Q(\cdot\mid c)$, then
$U\sim R(\cdot\mid c,\mathcal{T})$, and finally $(O,Z,L)$ from
$\Gamma_{\mathcal{T},U}$.  For any measurable set $A$ of target
continuations,
\begin{align}
    \Pr(Z\in A\mid c)
    &=\mathbb{E}_{\mathcal{T},U\mid c}
      [\Pr(Z\in A\mid c,\mathcal{T},U)]\\
    &=p(A\mid c).
\end{align}
Moreover, $O=Z_{1:L}$ almost surely, and $\{L\leq t\}$ remains
$\mathcal{H}_t$-measurable.  Thus, conditional non-anticipating exactness is
preserved after marginalizing over random tree construction and auxiliary
randomness.  In our setting, neither the drafter-constructed tree nor $U$ has
access to future target randomness.

\begin{lemma}[Round-wise composition]
\label{lem:round_composition}
At round $k$, let $c_k$ be the current context, $\mathcal{T}_k$ the realized
tree, $U_k$ the auxiliary control randomness, $Z^{(k)}$ the coupled target
continuation, and $L_k$ the output-block length, so
$O_k=Z^{(k)}_{1:L_k}$. Suppose each round satisfies
Eq.~\eqref{eq:app_conditional_coupling} with these round-indexed quantities and
$L_k$ is a stopping time with respect to
$\mathcal{H}_{k,t}:=\sigma(\mathcal{T}_k,U_k,Z^{(k)}_1,\ldots,Z^{(k)}_t)$.
Then repeatedly appending the returned blocks preserves the autoregressive
distribution induced by $p$.
\end{lemma}

\begin{proof}
For a realized round-$k$ output length $L_k=\ell$, the event $\{L_k=\ell\}$ is
measurable with respect to $\mathcal{H}_{k,\ell}$ and therefore depends only on
$(\mathcal{T}_k,U_k,Z^{(k)}_{1:\ell})$, not on $Z^{(k)}_{\ell+1:}$.  Because
$Z^{(k)}\mid(c_k,\mathcal{T}_k,U_k)\sim p(\cdot\mid c_k)$, conditioning on the
emitted block $O_k=Z^{(k)}_{1:\ell}$ and on the stopping event does not change the
target law of the unrevealed continuation beyond conditioning on the emitted
prefix itself. Restarting round $k+1$ from the updated context therefore uses
exactly the same conditional law as ordinary target decoding. Induction over
$k$ proves the claim.
\end{proof}

\subsection{Exit partition, probability conservation, and target coverage}
\label{app:exit_partition_proof}

\begin{proof}[Proof of Lemma~\ref{prop:exit_partition}]
Let $Z=(Z_1,Z_2,\ldots)\sim p(\cdot\mid c)$.
Because $\mathcal{T}$ is finite and prefix-closed, there is a unique deepest
node $V^\star(Z)$ such that $\pi(V^\star(Z))\preceq Z$.
If $V^\star(Z)$ is internal, the next target token cannot equal the draft
token of any child; otherwise the continuation would match a deeper node.
If $V^\star(Z)$ is a leaf, it is automatically the deepest matching node.
Hence the events
\begin{equation}
    E_v:=\{V^\star(Z)=v\},
    \qquad v\in\mathcal{V},
\end{equation}
are disjoint and exhaustive.

The continuation reaches $v$ with probability
\begin{equation}
    \Pr(\pi(v)\preceq Z\mid c,\mathcal{T})
    =p(\pi(v)\mid c)=P(v).
\end{equation}
Conditioned on reaching $v$, its next token is outside the drafted child-token
set $\mathcal{A}(v)$ with probability
\begin{equation}
    1-\sum_{u\in\Ch(v)}p_v(x_u)=\rho(v).
\end{equation}
Therefore,
\begin{equation}
    \Pr(E_v\mid c,\mathcal{T})
    =P(v)\rho(v)=w(v).
\end{equation}
Since the events $\{E_v\}_{v\in\mathcal{V}}$ form a partition,
\begin{equation}
    \sum_{v\in\mathcal{V}}w(v)=1.
\end{equation}
\end{proof}

The exit probabilities also satisfy a local probability-conservation identity:
\begin{align}
    w(v)+\sum_{u\in\Ch(v)}P(u)
    &=P(v)\rho(v)
      +\sum_{u\in\Ch(v)}P(v)p_v(x_u)\\
    &=P(v).
    \label{eq:app_local_conservation}
\end{align}
Thus, target probability reaching $v$ either exits at $v$ or continues through
one drafted child.  Summing Equation~\ref{eq:app_local_conservation} over the
subtree rooted at a node $u$ cancels every descendant prefix probability and
gives
\begin{equation}
    \sum_{v\succeq u} w(v)=P(u),
    \label{eq:app_subtree_conservation}
\end{equation}
where $v\succeq u$ means that $v$ lies in the subtree rooted at $u$.

Finally, $\depth(v)$ equals the number of non-root nodes on the path from the
root to $v$.  Using Equation~\ref{eq:app_subtree_conservation} and exchanging
the order of summation,
\begin{align}
    \sum_{v\in\mathcal{V}}w(v)\depth(v)
    &=\sum_{v\in\mathcal{V}}w(v)
      \sum_{u\in\mathcal{V}\setminus\{r\}}\1[v\succeq u]\\
    &=\sum_{u\in\mathcal{V}\setminus\{r\}}
      \sum_{v\succeq u}w(v)\\
    &=\sum_{u\in\mathcal{V}\setminus\{r\}}P(u).
\end{align}
This proves the two equivalent forms of target coverage in
Equation~\ref{eq:eal_two_decompositions}.

\subsection{Proof of the fixed-tree ceiling and canonical saturation}
\label{app:ceiling_uniqueness_proof}

\begin{proof}[Proof of Theorem~\ref{thm:unique_saturated_law}]
Let $\mathsf{A}$ be an exact path verifier.  Under the coupling defined in
Sec.~\ref{sec:setup}, let $Z\sim p(\cdot\mid c)$ be the target continuation and
let the verifier output be $O=Z_{1:L}$.  The ceiling argument uses only this
target-prefix identity, not the stopping-time property.  For every non-root
node $u$, define
\begin{align}
    A_u&:=\{u\text{ is included in the accepted draft path}\},\\
    F_u&:=\{\pi(u)\preceq Z\}.
\end{align}
Whenever $A_u$ occurs, the verifier output begins with $\pi(u)$ and is a prefix
of $Z$; therefore $A_u\subseteq F_u$ almost surely.  Consequently,
\begin{equation}
    \Pr_{\mathsf{A}}(A_u\mid c,\mathcal{T})
    \leq \Pr(F_u\mid c,\mathcal{T})
    =P(u),
    \label{eq:app_node_bound}
\end{equation}
which proves Equation~\ref{eq:retention_bound}.

An accepted root-to-node path contains exactly one non-root node for every
accepted draft token, so
\begin{equation}
    D_{\mathsf{A}}
    =\sum_{u\in\mathcal{V}\setminus\{r\}}\1[A_u].
\end{equation}
Taking expectations and applying Equation~\ref{eq:app_node_bound} gives
\begin{align}
    \mathbb{E}[D_{\mathsf{A}}\mid c,\mathcal{T}]
    &=\sum_{u\neq r}\Pr_{\mathsf{A}}(A_u\mid c,\mathcal{T})\\
    &\leq\sum_{u\neq r}P(u)
      =C_p(\mathcal{T}).
\end{align}
Since $L_{\mathsf{A}}=D_{\mathsf{A}}+1$, we also have
$\mathbb{E}[L_{\mathsf{A}}\mid c,\mathcal{T}]\leq1+C_p(\mathcal{T})$, proving
Equation~\ref{eq:eal_upper_bound}.

Now suppose that $\mathsf{A}$ is saturated.  For each non-root node, let
\begin{equation}
    \delta_u
    :=P(u)-\Pr_{\mathsf{A}}(A_u\mid c,\mathcal{T})\geq 0.
\end{equation}
Saturation implies $\sum_{u\neq r}\delta_u=0$, and hence every
$\delta_u=0$.  Thus
\begin{equation}
    \Pr_{\mathsf{A}}(A_u\mid c,\mathcal{T})=P(u)
    \qquad \text{for every }u\neq r.
    \label{eq:app_tight_retention}
\end{equation}
Since $A_u\subseteq F_u$ and both events have the same probability,
$A_u=F_u$ almost surely.  Because the draft tree is finite, these equalities
hold simultaneously for all nodes outside a single null event.

Let $A_r=F_r$ be the full sample space.  The accepted path terminates at $v$
exactly when $v$ is accepted and none of its children is accepted.  The child
acceptance events are disjoint, so
\begin{align}
    \Pr_{\mathsf{A}}(V=v\mid c,\mathcal{T})
    &=\Pr_{\mathsf{A}}(A_v\mid c,\mathcal{T})
      -\sum_{u\in\Ch(v)}
       \Pr_{\mathsf{A}}(A_u\mid c,\mathcal{T})\\
    &=P(v)-\sum_{u\in\Ch(v)}P(u)\\
    &=P(v)\rho(v)=w(v).
    \label{eq:app_saturated_exit_law}
\end{align}
Moreover, the simultaneous identities $A_u=F_u$ imply that $V$ equals the
target-induced exit node $V^\star(Z)$ almost surely.  Since the verifier output
satisfies $O=\pi(V)\mathbin{\Vert}Y=Z_{1:L}$ with $L=\depth(V)+1$, $Y$ is the target
token immediately after $\pi(V)$.  Therefore, for every internal node $v$ with
$w(v)>0$,
\begin{equation}
    \Pr(Y=a\mid V=v,c,\mathcal{T})
    =\frac{p_v(a)\1[a\notin\mathcal{A}(v)]}{\rho(v)}
    =p_v^{\mathrm{res}}(a).
\end{equation}
At a leaf, $\mathcal{A}(v)=\emptyset$ and $\rho(v)=1$, so the conditional law
is $p_v$.  This proves Equation~\ref{eq:canonical_output_law} and the uniqueness
of the accepted-path and bonus-token law at saturation.
\end{proof}

\subsection{Proof that the representative verifiers are saturated}
\label{app:verifier_equivalence_proof}

We consider the two target-local procedures defined in
Section~\ref{sec:background}.  Both operate on the same fixed draft tree and
the same target distributions $\{p_v\}$.  Sibling draft tokens are distinct.
For the target-induced flow, define
$L^\star:=\depth(V^\star(Z))+1$.  This is the first index at which the matched
node is a leaf or $Z_{L^\star}$ is not a drafted child token, so
$\{L^\star\leq t\}$ is determined by $Z_{1:t}$ and $L^\star$ is a stopping
time.  We show that each procedure realizes the joint target-induced exit and
bonus-token law; coupling its output to
$Z_{1:L^\star}=\pi(V^\star(Z))\mathbin{\Vert}Y^\star(Z)$ therefore establishes
non-anticipating exactness as well as saturation.

\subsubsection{Predraw-and-follow}

At every node $v$, predraw-and-follow independently draws a token
$R_v\sim p_v$.  Starting from the root, it follows child $u\in\Ch(v)$ if
$R_v=x_u$; otherwise it exits at $v$ and returns $R_v$ as the bonus token.

Fix a node $v$ and write its root-to-node path as
$r=v_0,v_1,\ldots,v_d=v$.  The procedure reaches $v$ exactly when
$D_{v_{i-1}}=x_{v_i}$ for all $i=1,\ldots,d$.  Hence
\begin{align}
    \Pr(\text{reach }v\mid c,\mathcal{T})
    &=\prod_{i=1}^{d}p_{v_{i-1}}(x_{v_i})
      =P(v).
\end{align}
Conditioned on reaching $v$, the independent draw $R_v$ is outside the drafted
child-token set with probability
\begin{equation}
    1-\sum_{u\in\Ch(v)}p_v(x_u)=\rho(v).
\end{equation}
Therefore,
\begin{equation}
    \Pr(\text{exit at }v\mid c,\mathcal{T})
    =P(v)\rho(v)=w(v).
\end{equation}
Conditioned on an internal exit, the bonus token follows
$p_v^{\mathrm{res}}$; at a leaf it follows $p_v$.  Thus the accepted-path and
bonus-token law matches the target-induced law, so predraw-and-follow is exact.
Its expected accepted draft-token count is
$\sum_v w(v)\depth(v)=C_p(\mathcal{T})$, so its expected output-block length is
$1+C_p(\mathcal{T})$ and it is saturated.

\subsubsection{Exact sequential residual verification}

Fix an internal node $v$ reached by the verifier.  Let its drafted children be
tested in an arbitrary order $u_1,\ldots,u_k$, with draft tokens
$x_1,\ldots,x_k$.  Define
\begin{equation}
    p_j:=p_v(x_j),
    \qquad
    M_j:=\sum_{\ell=1}^{j}p_\ell,
    \qquad
    M_0:=0.
\end{equation}
Let $g^{(0)}:=p_v$.  If child $u_j$ is rejected and the rejection event has
positive probability, exact residual renormalization removes token $x_j$:
\begin{equation}
    g^{(j)}(a)
    :=\frac{g^{(j-1)}(a)\1[a\neq x_j]}
    {1-g^{(j-1)}(x_j)}.
    \label{eq:app_residual_update}
\end{equation}
At step $j$, the verifier follows $u_j$ with conditional probability
$g^{(j-1)}(x_j)$.  If all children are rejected, it draws the bonus token from
$g^{(k)}$.  A leaf has no child tests and returns a bonus token from $p_v$.

For every $j$ with $M_j<1$, induction gives
\begin{equation}
    g^{(j)}(a)=\frac{p_v(a)}{1-M_j},
    \qquad a\notin\{x_1,\ldots,x_j\}.
    \label{eq:app_residual_induction}
\end{equation}
Indeed, assuming the identity at $j-1$,
\begin{align}
    g^{(j)}(a)
    &=\frac{p_v(a)/(1-M_{j-1})}
    {1-p_j/(1-M_{j-1})}
      =\frac{p_v(a)}{1-M_j}.
\end{align}
If $M_j=1$, rejecting the first $j$ children has probability zero, so no later
residual state affects the output.

Before child $u_j$ is tested,
\begin{equation}
    g^{(j-1)}(x_j)=\frac{p_j}{1-M_{j-1}}.
\end{equation}
The probability of rejecting the first $j-1$ children telescopes to
\begin{align}
    \prod_{i=1}^{j-1}
    \left(1-g^{(i-1)}(x_i)\right)
    &=\prod_{i=1}^{j-1}\frac{1-M_i}{1-M_{i-1}}
      =1-M_{j-1}.
\end{align}
Thus, the unconditional probability of following child $u_j$ is
\begin{equation}
    (1-M_{j-1})\frac{p_j}{1-M_{j-1}}
    =p_j=p_v(x_j).
\end{equation}
Similarly, the probability of rejecting every child is
\begin{equation}
    \prod_{i=1}^{k}\frac{1-M_i}{1-M_{i-1}}
    =1-M_k=\rho(v),
\end{equation}
and, whenever $\rho(v)>0$, the final residual target distribution is
\begin{equation}
    g^{(k)}(a)
    =\frac{p_v(a)\1[a\notin\mathcal{A}(v)]}{\rho(v)}
    =p_v^{\mathrm{res}}(a).
\end{equation}
Hence the local transition kernel is independent of the order of the drafted
children: each child $u$ is followed with probability $p_v(x_u)$, while the
remaining probability $\rho(v)$ exits with bonus-token law
$p_v^{\mathrm{res}}$.  Multiplying the child-transition probabilities along a
root-to-$v$ path gives reach probability $P(v)$ and exit probability $w(v)$.
Thus the accepted-path and bonus-token law matches the target-induced law, so
the verifier is exact. Its expected accepted draft-token count is
$C_p(\mathcal{T})$, its expected output-block length is $1+C_p(\mathcal{T})$,
and hence it is saturated.

\subsection{Why the fixed-tree result does not contradict HSD}
\label{app:sampled_chain_relation}

The fixed-tree ceiling applies after a draft tree has already been realized.
HSD~\citep{zhou2026hsd} studies a different probability structure: a draft
chain $x_{1:k}$ is sampled from draft distributions $q_1,\ldots,q_k$, and
verification retains the target-to-draft ratios
\begin{equation}
    r_t:=\frac{p_t(x_t)}{q_t(x_t)}.
\end{equation}
For a realized chain, position-wise speculative verification accepts the full
chain with probability
\begin{equation}
    \prod_{t=1}^{k}\min\{1,r_t\}
    \leq
    \min\!\left\{1,\prod_{t=1}^{k}r_t\right\}.
\end{equation}
When ratios above and below one occur along the same chain, position-wise
clipping prevents surplus at one position from compensating for a deficit at
another; hierarchical chain verification can recover part of this joint
compensation.

Our analysis instead conditions on a realized draft tree.  At that point, the
draft model affects verification through which draft tokens are present in the
tree, while the fixed-tree ceiling follows from target prefix events.  Even if
a verifier also consults draft scores, no exact path verifier can accept a draft
token absent from the realized tree.  The two results therefore address
different stages of speculative decoding and are complementary rather than
contradictory.

\section{Correctness of Tree Exit Verification}
\label{app:tev}

\algtext*{EndFor}
\begin{algorithm}[t]
    \caption{Tree Exit Verification (\TEV) for speculative verification}
    \label{alg:tev}
    \begin{algorithmic}[1]
        \Require Fixed tree $\mathcal{T}=(\mathcal{V},\mathcal{E},r)$ and target distributions $\{p_v\}_{v\in\mathcal{V}}$
        \Ensure Accepted path followed by one bonus token
        \State $P(r)\gets 1$
        \For{each depth $d=0,1,\ldots$}
            \ForAll{$v$ at depth $d$ \textbf{in parallel}}
                \State $\rho(v)\gets1-\sum_{u\in\Ch(v)}p_v(x_u)$; $w(v)\gets P(v)\rho(v)$
                \State $P(u)\gets P(v)p_v(x_u)$ for all $u\in\Ch(v)$
            \EndFor
        \EndFor
        \State Draw $V\sim\mathrm{Categorical}(\{w(v)\}_{v\in\mathcal{V}})$
        \State Draw $Y\sim p_V$ if $V$ is a leaf, and otherwise $Y\sim p_V^{\mathrm{res}}$
        \State \Return $\pi(V)\mathbin{\Vert}Y$
    \end{algorithmic}
\end{algorithm}

\subsection{Proof of exactness and saturation of \TEV}
\label{app:tev_exactness_proof}

\begin{proof}[Proof of Corollary~\ref{prop:tev_retention}]
Fix a context $c$ and a realized finite, prefix-closed draft tree
$\mathcal{T}$.  Let $Z\sim p(\cdot\mid c)$, let $V^\star(Z)$ be its
target-induced exit node, and let $Y^\star(Z)$ be the target token immediately
following $\pi(V^\star(Z))$.

By Lemma~\ref{prop:exit_partition},
$\Pr(V^\star(Z)=v\mid c,\mathcal{T})=w(v)$.  Conditioned on an internal exit
at $v$ with $w(v)>0$, the next target token lies outside
$\mathcal{A}(v)$ and therefore follows $p_v^{\mathrm{res}}$; conditioned on a
leaf exit, it follows $p_v$.  These are exactly the two distributions used by
Algorithm~\ref{alg:tev}.  Hence
\begin{equation}
    (V,Y)_{\TEV}
    \overset{d}{=}
    \bigl(V^\star(Z),Y^\star(Z)\bigr).
    \label{eq:app_tev_joint_law}
\end{equation}
We may therefore couple the two pairs to be equal almost surely.  Under this
coupling, set $L:=\depth(V)+1=\depth(V^\star(Z))+1$.  This is the first index at
which the matched node is a leaf or the next target token leaves the drafted
child set. Hence $\{L\leq t\}$ is determined by $Z_{1:t}$ and is therefore
measurable with respect to $\mathcal{H}_t=\sigma(\mathcal{T},U,Z_{1:t})$.
Moreover,
\begin{equation}
    O=\pi(V)\mathbin{\Vert}Y=Z_{1:L}\qquad\text{almost surely}.
\end{equation}
Thus \TEV{} satisfies the non-anticipating exactness definition (without using
any additional early-stopping randomness $U$).

For any non-root node $u$, \TEV{} accepts $u$ exactly when its exit node lies
in the subtree rooted at $u$.  Using
Equation~\ref{eq:app_subtree_conservation},
\begin{equation}
    \Pr_{\TEV}(u\text{ is accepted}\mid c,\mathcal{T})
    =\sum_{v\succeq u}w(v)=P(u).
    \label{eq:app_tev_node_retention}
\end{equation}
Thus \TEV{} attains every node-wise bound in
Theorem~\ref{thm:unique_saturated_law}.  Finally,
\begin{align}
    \mathbb{E}[D_{\TEV}\mid c,\mathcal{T}]
    &=\sum_{u\in\mathcal{V}\setminus\{r\}}
      \Pr_{\TEV}(u\text{ is accepted}\mid c,\mathcal{T})\\
    &=\sum_{u\in\mathcal{V}\setminus\{r\}}P(u)
      =C_p(\mathcal{T}),\\
    \mathbb{E}[L_{\TEV}\mid c,\mathcal{T}]
    &=1+C_p(\mathcal{T}).
\end{align}
Therefore \TEV{} attains the fixed-tree ceiling and is saturated.  If the
draft tree is random, the same conclusion follows by conditioning on the
realized tree and applying Appendix~\ref{app:random_tree_conditioning}.
\end{proof}

\section{Additional Details for Exit-Guided Training}
\label{app:exit_training_details}

\subsection{One-step target-coverage gain}
\label{app:one_step_gain}

Let $v$ be an expandable node and let $a\notin\mathcal{A}(v)$ be a missing
draft token.  Write $u=(v,a)$ for the new child and
$\mathcal{T}^{+}=\mathcal{T}\oplus(v,a)$ for the tree obtained by adding only
this node.  The target prefix probabilities of all existing nodes are
unchanged, and the new node has
\begin{equation}
    P(u)=P(v)p_v(a).
\end{equation}
Therefore, by the definition of target coverage,
\begin{align}
    C_p(\mathcal{T}^{+})-C_p(\mathcal{T})
    &=P(u)=P(v)p_v(a),
\end{align}
which proves Equation~\ref{eq:one_step_marginal_gain}.  Summing over every
missing child token gives
\begin{align}
    \sum_{a\notin\mathcal{A}(v)}P(v)p_v(a)
    &=P(v)\left(1-\sum_{u\in\Ch(v)}p_v(x_u)\right)\\
    &=P(v)\rho(v)=w(v),
\end{align}
which proves Equation~\ref{eq:exit_as_expansion_gain}.  This is a gross
one-step gain before enforcing the fixed tree budget; a budget-preserving
update may also need to remove another node.

\subsection{Efficient evaluation of the exit-guided loss}
\label{app:efficient_loss}

Let $\mathcal{N}$ denote the multiset of non-root draft-tree nodes in a
minibatch, let $c(v)$ denote the training context associated with node $v$,
and let the detached weight be $\omega_v=w(v)$.  Suppose that the incoming
draft token $x_v$ is scored by the DFlash position-wise draft-logit row
$z_{\mathrm{row}(v)}\in\mathbb{R}^{|\mathcal{X}|}$, where
$\mathrm{row}(v)$ indexes the pair consisting of $v$'s minibatch example and
future position $\depth(v)$.  Within one training example, this row is shared
by all tree nodes at the same depth and is not conditioned on
$\pi(\pa(v))$.  Then
\begin{align}
    &\sum_{v\in\mathcal{N}}
      \omega_v\bigl[-\log q_{\theta,\mathrm{row}(v)}(x_v\mid c(v))\bigr]\nonumber\\
    &\quad=
      \sum_{v\in\mathcal{N}}\omega_v
      \left[
      \operatorname{logsumexp}\!\left(z_{\mathrm{row}(v)}\right)
      -z_{\mathrm{row}(v),x_v}
      \right]\\
    &\quad=
      \sum_m
      \left(
      \sum_{v\in\mathcal{N}:\,\mathrm{row}(v)=m}\omega_v
      \right)
      \operatorname{logsumexp}(z_m)
      -\sum_{v\in\mathcal{N}}
      \omega_v z_{\mathrm{row}(v),x_v}.
    \label{eq:partition_target_decomposition}
\end{align}
Thus, each distinct draft-logit row requires only one log-normalizer, and the
second term uses a gather over the selected draft-token logits.  The minibatch
normalization in Equation~\ref{eq:tree_auxiliary_loss} is applied after this
numerator is computed.  This decomposition reuses the batched draft-model
language head and avoids materializing a separate normalized vocabulary
distribution for every selected tree node.

\section{Additional Experimental Details}
\label{app:experimental_details}

\subsection{Verifier-stage efficiency}
\label{app:verifier_efficiency}

For each budget $B\in\{32,64,128\}$, we hold the drafter, builder, and realized tree fixed across verifiers and time only verifier decision and cache commit, excluding the shared target-model tree forward pass. Corollary~\ref{thm:verifier_equivalence} guarantees identical fixed-tree expected acceptance for Predraw and TEV at every budget.

\begin{table*}[t]
    \caption{Verifier-stage latency per decoding cycle on Qwen3-4B across tree budgets. The measured interval covers verifier decision and cache commit and excludes the shared target-model tree forward pass. For each $B$, both verifiers use the same drafter, builder, and realized trees.}
    \label{tab:verifier_latency}
    \vskip 0.08in
    \begin{center}
    \begin{small}
    \resizebox{\textwidth}{!}{
    \setlength{\tabcolsep}{4pt}
    \begin{tabular}{c l||cc|cc|cc|cc||cc}
        \toprule
        & & \multicolumn{2}{c|}{\textbf{MT-Bench}} & \multicolumn{2}{c|}{\textbf{HumanEval}} & \multicolumn{2}{c|}{\textbf{GSM8K}} & \multicolumn{2}{c||}{\textbf{MATH-500}} & \multicolumn{2}{c}{\textbf{Average}} \\
        \cmidrule{3-12}
        $B$ & \textbf{Verifier} & ms$\downarrow$ & red. & ms$\downarrow$ & red. & ms$\downarrow$ & red. & ms$\downarrow$ & red. & ms$\downarrow$ & red. \\
        \midrule
        \multicolumn{12}{c}{\textit{Temperature = 0.6}} \\
        \midrule
        \multirow{2}{*}{32} & Predraw & 1.371 & --- & 1.435 & --- & 1.428 & --- & 1.393 & --- &  1.407 & --- \\
          & \textbf{TEV} & \textbf{1.204} & \textbf{12.2\%} & \textbf{1.245} & \textbf{13.2\%} & \textbf{1.239} & \textbf{13.2\%} & \textbf{1.211} & \textbf{13.0\%} & \textbf{1.225} & \textbf{12.9\%} \\
        \cmidrule{1-12}
        \multirow{2}{*}{64} & Predraw & 1.392 & --- & 1.487 & --- & 1.443 & --- & 1.488 & --- & 1.453 & --- \\
          & \textbf{TEV} & \textbf{1.212} & \textbf{12.9\%} & \textbf{1.268} & \textbf{14.7\%} & \textbf{1.193} & \textbf{17.3\%} & \textbf{1.257} & \textbf{15.5\%} & \textbf{1.232} & \textbf{15.2\%} \\
        \cmidrule{1-12}
        \multirow{2}{*}{128} & Predraw & 1.620 & --- & 1.719 & --- & 1.681 & --- & 1.616 & --- & 1.659 & --- \\
          & \textbf{TEV} & \textbf{1.246} & \textbf{23.1\%} & \textbf{1.265} & \textbf{26.4\%} & \textbf{1.259} & \textbf{25.1\%} & \textbf{1.248} & \textbf{22.8\%} & \textbf{1.254} & \textbf{24.3\%} \\
        \midrule
        \multicolumn{12}{c}{\textit{Temperature = 1.0}} \\
        \midrule
        \multirow{2}{*}{32} & Predraw & 1.337 & --- & 1.368 & --- & 1.348 & --- & 1.411 & --- & 1.366 & --- \\
          & \textbf{TEV} & \textbf{1.187} & \textbf{11.2\%} & \textbf{1.235} & \textbf{9.7\%} & \textbf{1.184} & \textbf{12.1\%} & \textbf{1.197} & \textbf{15.2\%} & \textbf{1.201} & \textbf{12.0\%} \\
        \cmidrule{1-12}
        \multirow{2}{*}{64} & Predraw & 1.416 & --- & 1.428 & --- & 1.398 & --- & 1.442 & --- & 1.421 & --- \\
          & \textbf{TEV} & \textbf{1.200} & \textbf{15.3\%} & \textbf{1.240} & \textbf{13.2\%} & \textbf{1.207} & \textbf{13.7\%} & \textbf{1.240} & \textbf{14.0\%} & \textbf{1.222} & \textbf{14.0\%} \\
        \cmidrule{1-12}
        \multirow{2}{*}{128} & Predraw & 1.549 & --- & 1.545 & --- & 1.535 & --- & 1.509 & --- & 1.534 & --- \\
          & \textbf{TEV} & \textbf{1.266} & \textbf{18.3\%} & \textbf{1.276} & \textbf{17.3\%} & \textbf{1.269} & \textbf{17.3\%} & \textbf{1.277} & \textbf{15.4\%} & \textbf{1.272} & \textbf{17.1\%} \\
        \bottomrule
    \end{tabular}}
    \end{small}
    \end{center}
    \vskip -0.10in
\end{table*}

TEV reduces verifier-stage latency on every benchmark and budget. Averaged over tasks, the reduction grows from 12.9\% to 15.2\% to 24.3\% as $B$ increases from 32 to 64 to 128 at $T=0.6$; the corresponding reductions at $T=1.0$ are 12.0\%, 14.0\%, and 17.1\%. Because the drafter, realized tree, and target forward pass are fixed within each comparison, these measurements isolate execution of the common canonical law. The widening gap with tree size is consistent with Predraw's per-node sampling cost, whereas TEV uses level-wise probability propagation followed by one exit-node and one bonus-token draw. Even so, the stage-level reduction is larger than the end-to-end improvement in Table~\ref{tab:crossed_drafter_verifier} because verifier decision and cache commit constitute only a small fraction of total decoding time.

\subsection{Crossed drafter--verifier evaluation}
\label{app:crossed_experiment}

Table~\ref{tab:crossed_drafter_verifier} crosses the drafter checkpoint with the verification procedure while holding the DDTree builder, $B=64$ budget, target model, decoding distribution, and runtime configuration fixed. The comparison reports end-to-end speedup and output-block length at $T=1$; verifier-stage latency is isolated separately in Table~\ref{tab:verifier_latency}. We omit $T=0$ because its degenerate target distribution reduces verification to greedy decoding rather than the sampling regime for which TEV is designed.

\begin{table*}[t]
    \caption{Crossed drafter--verifier evaluation on Qwen3-4B at $T=1$. Each cell reports end-to-end speedup ratio ($SR$) and average output-block length including the bonus token ($\tau$). The builder and tree budget $B=64$ are fixed across all four combinations.}
    \label{tab:crossed_drafter_verifier}
    \vskip 0.08in
    \begin{center}
    \begin{small}
    \resizebox{\textwidth}{!}{
    \setlength{\tabcolsep}{4pt}
    \begin{tabular}{ll||cc|cc|cc|cc||cc}
        \toprule
        & & \multicolumn{2}{c|}{\textbf{MT-Bench}} & \multicolumn{2}{c|}{\textbf{HumanEval}} & \multicolumn{2}{c|}{\textbf{GSM8K}} & \multicolumn{2}{c||}{\textbf{MATH-500}} & \multicolumn{2}{c}{\textbf{Average}} \\
        \cmidrule{3-12}
        \textbf{Drafter} & \textbf{Verifier} & $SR\uparrow$ & $\tau\uparrow$ & $SR\uparrow$ & $\tau\uparrow$ & $SR\uparrow$ & $\tau\uparrow$ & $SR\uparrow$ & $\tau\uparrow$ & $SR\uparrow$ & $\tau\uparrow$ \\
        \midrule
        \multirow{2}{*}{Original DFlash}
        & Predraw & 3.44 & 5.04 & 5.22 & 7.53 & 5.48 & 7.85 & 5.65 & 7.88 & 4.95 & 7.08 \\
        & TEV & 3.47 & 5.07 & 5.22 & 7.52 & 5.51 & 7.88 & 5.67 & 7.90 & 4.97 & 7.09 \\
        \cmidrule{1-12}
        \multirow{2}{*}{Exit-trained DFlash}
        & Predraw & 3.95 & 5.86 & 5.61 & 8.05 & 6.44 & 9.23 & 6.22 & 9.06 & 5.55 & 8.05 \\
        & \textbf{TEV} & \textbf{3.98} & \textbf{5.86} & \textbf{5.64} & \textbf{8.06} & \textbf{6.47} & \textbf{9.25} & \textbf{6.26} & \textbf{9.08} & \textbf{5.59} & \textbf{8.06} \\
        \bottomrule
    \end{tabular}}
    \end{small}
    \end{center}
    \vskip -0.10in
\end{table*}

For a fixed drafter, the empirical $\tau$ differences between Predraw and TEV are at most 0.03 tokens, as expected for separate finite-sample runs of two saturated verifiers with the same expected output-block length. TEV nevertheless improves average $SR$ from $4.95\times$ to $4.97\times$ with the original drafter and from $5.55\times$ to $5.59\times$ with the exit-trained drafter. In contrast, changing the drafter has the dominant effect: under Predraw, exit training increases average $\tau$ by 13.7\% and $SR$ by 12.1\%; under TEV, the gains are 13.7\% and 12.5\%. This separates improved tree allocation from the smaller verifier-execution gain.

\subsection{Tree-budget robustness}
\label{app:budget_ablation}

Table~\ref{tab:budget_ablation} varies only the inference-time DDTree budget. All drafter checkpoints are fixed; in particular, ExitTrain is trained with $B=64$ and evaluated without retraining at $B\in\{32,64,128\}$. We report each benchmark separately together with the four-benchmark average.

\begin{table*}[t]
    \caption{Qwen3-4B performance across tree budgets on MT-Bench, HumanEval, GSM8K, and MATH-500. ExitTrain is trained at $B=64$ and evaluated at every budget without retraining.}
    \label{tab:budget_ablation}
    \vskip 0.08in
    \begin{center}
    \begin{small}
    \resizebox{\textwidth}{!}{
    \setlength{\tabcolsep}{4pt}
    \begin{tabular}{cl||cc|cc|cc|cc||cc}
        \toprule
        & & \multicolumn{2}{c|}{\textbf{MT-Bench}} & \multicolumn{2}{c|}{\textbf{HumanEval}} & \multicolumn{2}{c|}{\textbf{GSM8K}} & \multicolumn{2}{c||}{\textbf{MATH-500}} & \multicolumn{2}{c}{\textbf{Average}} \\
        \cmidrule{3-12}
        $B$ & \textbf{Method} & $SR\uparrow$ & $\tau\uparrow$ & $SR\uparrow$ & $\tau\uparrow$ & $SR\uparrow$ & $\tau\uparrow$ & $SR\uparrow$ & $\tau\uparrow$ & $SR\uparrow$ & $\tau\uparrow$ \\
        \midrule
        \multicolumn{12}{c}{\textit{Temperature = 0}} \\
        \midrule
        \multirow{3}{*}{32} & DDTree & 3.69 & 5.19 & 5.47 & 7.74 & 5.40 & 8.03 & 6.50 & 8.74 & 5.26 & 7.43 \\
          & Draft-OPD & 3.91 & 5.32 & 5.60 & 7.82 & 5.87 & 8.92 & 6.84 & 9.70 & 5.56 & 7.94 \\
          & \textbf{ExitTrain + TEV} & \textbf{4.11} & \textbf{5.87} & \textbf{5.72} & \textbf{8.22} & \textbf{6.68} & \textbf{9.78} & \textbf{6.94} & \textbf{9.77} & \textbf{5.86} & \textbf{8.41} \\
        \cmidrule{1-12}
        \multirow{3}{*}{64} & DDTree & 3.81 & 5.41 & 5.71 & 8.13 & 5.65 & 8.48 & 6.83 & 9.16 & 5.50 & 7.80 \\
        & Draft-OPD & 4.16 & 5.73 & 5.96 & 8.35 & 6.11 & 9.30 & 7.02 & 10.14 & 5.81 & 8.38 \\
        & \textbf{ExitTrain + TEV} & \textbf{4.28} & \textbf{6.17} & \textbf{6.09} & \textbf{8.74} & \textbf{6.91} & \textbf{10.22} & \textbf{7.23} & \textbf{10.26} & \textbf{6.13} & \textbf{8.85} \\
        \cmidrule{1-12}
        \multirow{3}{*}{128} & DDTree & 3.95 & 5.77 & 5.93 & 8.57 & 5.86 & 8.82 & 6.89 & 9.60 & 5.66 & 8.19 \\
          & Draft-OPD & 4.25 & 6.10 & 6.10 & 8.64 & 6.28 & 9.63 & 7.14 & 10.40 & 5.94 & 8.69 \\
          & \textbf{ExitTrain + TEV} & \textbf{4.40} & \textbf{6.50} & \textbf{6.28} & \textbf{9.10} & \textbf{7.15} & \textbf{10.60} & \textbf{7.43} & \textbf{10.72} & \textbf{6.31} & \textbf{9.23} \\
        \midrule
        \multicolumn{12}{c}{\textit{Temperature = 1}} \\
        \midrule
        \multirow{3}{*}{32} & DDTree & 3.31 & 4.82 & 5.00 & 7.15 & 5.24 & 7.50 & 5.39 & 7.60 & 4.73 & 6.77  \\
          & Draft-OPD & 3.56 & 4.97 & 5.34 & 7.49 & 5.58 & 8.12 & 5.86 & 8.14 & 5.09 & 7.18  \\
          & \textbf{ExitTrain + TEV} & \textbf{3.79} & \textbf{5.50} & \textbf{5.38} & \textbf{7.75} & \textbf{6.09} & \textbf{8.78} & \textbf{6.02} & \textbf{8.52} & \textbf{5.32} & \textbf{7.64} \\
        \cmidrule{1-12}
        \multirow{3}{*}{64} & DDTree & 3.44 & 5.04 & 5.22 & 7.53 & 5.48 & 7.85 & 5.65 & 7.88 & 4.95 & 7.08 \\
        & Draft-OPD & 3.74 & 5.26 & 5.49 & 7.80 & 5.88 & 8.58 & 6.08 & 8.50 & 5.30 & 7.54 \\
        & \textbf{ExitTrain + TEV} & \textbf{3.98} & \textbf{5.86} & \textbf{5.64} & \textbf{8.06} & \textbf{6.47} & \textbf{9.25} & \textbf{6.26} & \textbf{9.08} & \textbf{5.59} & \textbf{8.06} \\
        \cmidrule{1-12}
        \multirow{3}{*}{128} & DDTree & 3.50 & 5.32 & 5.42 & 7.93 & 5.74 & 8.35 & 5.75 & 8.38 & 5.10 & 7.50 \\
          & Draft-OPD & 3.80 & 5.54 & 5.81 & 8.31 & 5.99 & 8.89 & 6.17 & 8.78 & 5.44 & 7.88  \\
          & \textbf{ExitTrain + TEV} & \textbf{4.06} & \textbf{6.08} & \textbf{5.77} & \textbf{8.45} & \textbf{6.60} & \textbf{9.66} & \textbf{6.36} & \textbf{9.30} & \textbf{5.70} & \textbf{8.37} \\
        \bottomrule
    \end{tabular}}
    \end{small}
    \end{center}
    \vskip -0.10in
\end{table*}

Increasing the budget raises absolute output-block length for every method. ExitTrain + TEV improves average $\tau$ from 8.41 to 8.85 to 9.23 at $T=0$ and from 7.64 to 8.06 to 8.37 at $T=1$ as $B$ grows from 32 to 64 to 128; average $SR$ increases from $5.86\times$ to $6.13\times$ to $6.31\times$ and from $5.32\times$ to $5.59\times$ to $5.70\times$, respectively. More importantly, its advantage persists at every budget. Relative to DDTree, the gains span 11.6\%--13.8\% in $\tau$ and 11.4\%--12.9\% in $SR$; relative to Draft-OPD, they span 5.6\%--6.9\% and 4.5\%--6.2\%. Thus, the drafter trained with $B=64$ transfers to both smaller and larger inference trees without budget-specific retraining.

Unless otherwise stated, the main and crossed experiments use a tree budget of $B=64$ non-root nodes, BF16 model weights, SDPA for target-model attention, FlashAttention-2 for the drafter, batch size one, and at most 2,048 generated tokens. Every evaluation configuration uses one NVIDIA H20-96G GPU. We evaluate 128 GSM8K examples, all 164 HumanEval problems, 128 MATH-500 examples, and 80 MT-Bench questions; GSM8K, HumanEval, and MATH-500 are repeated three times and MT-Bench twice. Qwen3 models use their chat template with thinking disabled. Speedup ratios are computed against autoregressive decoding measured under the same runtime configuration as the corresponding speculative run.

\paragraph{Training configuration.}
All runs initialize from the official DFlash checkpoints released at \url{https://huggingface.co/collections/z-lab/dflash}. We train online with SpecForge using prompts from the \texttt{open-perfectblend} mixture~\citep{xu2024perfectblendredefiningrlhf}: the frozen target model generates rollout responses during training, and the drafter is updated on these target-generated trajectories. Training uses one 8$\times$H20-96G server, with one GPU assigned to target-model rollout and the remaining seven GPUs assigned to draft-model training. All models use learning rate $10^{-4}$ and batch size 4. The measured wall-clock training times on this server are approximately 35 hours for Qwen3-4B, 55 hours for Qwen3-8B, 100 hours for Gemma-4-12B-IT, and 80 hours for Qwen3-Coder-30B-A3B. Every training-objective comparison in Sec.~\ref{sec:ablation} uses the same initialization, online data pipeline, and compute budget.

\end{document}